\documentclass[conference]{IEEEtran}
\IEEEoverridecommandlockouts

\usepackage{cite}
\usepackage{amsmath,amssymb,amsfonts}
\usepackage{amsthm}
\usepackage[ruled,lined]{algorithm2e}
\usepackage{graphicx}
\usepackage{booktabs}
\usepackage{float}
\usepackage{url}
\usepackage[hidelinks]{hyperref}
\hypersetup{%
  pdftitle={Accelerating Visual Policy Learning with Sampling-Based Model Predictive Control},
  pdfauthor={Yilang Liu, Haoxiang You, Qian Wang, Daniel Rakita, Ian Abraham}}

\newtheorem{theorem}{Theorem}
\newtheorem{proposition}[theorem]{Proposition}

\begin{document}

\bstctlcite{SGPS:BSTcontrol}

\title{Accelerating Visual Policy Learning with Sampling-Based Model Predictive Control}

\author{Yilang Liu$^{1}$, Haoxiang You$^{1}$, Qian Wang$^{2}$, Daniel Rakita$^{2}$, and Ian Abraham$^{3}$
\thanks{$^{1}$Yilang Liu and Haoxiang You are with the Department of Mechanical Engineering, Yale University, New Haven, CT 06520, USA.
        {\tt\small \{yilang.liu, haoxiang.you\}@yale.edu}}%
\thanks{$^{2}$Qian Wang and Daniel Rakita are with the Department of Computer Science, Yale University, New Haven, CT 06520, USA.
        {\tt\small \{peter.wang.qw262, daniel.rakita\}@yale.edu}}%
\thanks{$^{3}$Ian Abraham is with the School of Electrical and Computer Engineering, University of Sydney, Sydney, NSW 2006, Australia.
        {\tt\small ian.abraham@sydney.edu.au}}%
}

\maketitle

\begin{abstract}
Learning visual policies for locomotion and manipulation requires coordinating contact with the environment and can incur substantial computation and GPU memory costs.
First-order policy gradients (FoPG) reduce training cost through differentiable simulation, but local optimization can converge to unintended contact patterns.
To address this shortfall, we propose Sampling-Guided Policy Search (SGPS), which couples recurring action-target refinement by sampling-based model-predictive control with first-order policy optimization.
Behavior cloning initializes the policy from sampled actions; training then alternates sampling-based refinement with short-horizon FoPG updates under perturbed initial states and randomized dynamics.
For visual policy training, we use a decoupled FoPG formulation that excludes rendering from the computation graph, enabling direct learning from depth observations without a state-policy teacher.
On a single GPU, SGPS learns policies for locomotion, obstacle traversal, crate pushing, and bimanual carrying on simulated Unitree Go2 and G1 robots.
Our experiments further show that refinement improves policy learning beyond initialization and tracking alone.
For hardware deployment, the distilled policy transfers zero-shot to a real Go2 and uses onboard depth to autonomously trot, crawl, clear hurdles, and switch between these behaviors.
\end{abstract}

\section{Introduction}\label{sec: intro}

Reinforcement learning has enabled legged locomotion over rough terrain~\cite{you2026efficientonpolicyvisualrlstochastic, lee2020learning} and perception-based obstacle traversal~\cite{zhuang2023parkour, cheng2024extreme, hoeller2024anymal}, supporting deployment in field tasks such as autonomous exploration and coverage~\cite{hughes2026asymptoticallyoptimalergodiccoverage}.
However, training these policies can still require large numbers of simulation steps collected in parallel environments~\cite{schulman2017ppo, makoviychuk2021isaac, rudin2022learning}, with additional computation and GPU memory for visual observations~\cite{you2025dva}.
To simplify visual policy training, recent works first learn a teacher policy using privileged state information, then train a student visual policy using onboard observations under the supervision of the teacher policy~\cite{lee2020learning, miki2022learning, kumar2021rma, chen2019lbc}.

To reduce training cost, methods such as SHAC~\cite{xu2021shac} compute first-order policy gradients (FoPG) through differentiable simulation~\cite{song2024learning, xing2024stabilizing}.
Building on this approach, D.Va~\cite{you2025dva} supports visual observations by decoupling rendering from the gradient computation.
However, local updates can produce unintended contact patterns when rewards do not distinguish gaits.
Reference behaviors can guide these updates toward the intended contacts.

A promising source of these references is sampling-based model-predictive control (MPC), which uses derivative-free optimization~\cite{williams2017information, howell2022predictivesamplingrealtimebehaviour, pinneri2020icem}.
For example, DIAL-MPC~\cite{xue2024dialmpc} generates locomotion through annealed sampling without prior policy training.
These rollouts are dynamically feasible under the planning model, but deploying MPC requires online optimization and state estimation.
The remaining challenge is to translate these trajectories into visual policies that tolerate deviations and execute without online MPC.

To address this challenge, we propose Sampling-Guided Policy Search (SGPS), centered on recurring sampling-based action-target refinement coupled with decoupled FoPG (Fig.~\ref{fig: teaser}).
MPC supplies BC initialization and a fixed tracking reference; refinement then generates new action targets to supplement local FoPG updates under state perturbations and randomized dynamics.
For visual learning, decoupled FoPG trains specialists from depth and tracking observations without a state-policy teacher.
For deployment, we distill the specialists into a single visual policy that transfers zero-shot to hardware and autonomously switches between behaviors using onboard observations and a velocity command, without online MPC or reference inputs.

Our contributions are:
\begin{itemize}
    \item SGPS, a policy-learning framework that combines recurring sampling-based action refinement with decoupled FoPG, refreshing action targets while keeping the tracking reference fixed.
    \item Direct training of visual policies with sampling-based supervision, without a separate state-policy teacher.
    \item Zero-shot deployment of a unified visual policy on Go2, with autonomous behavior transitions.
\end{itemize}

\begin{figure}[tb]
    \centering
    \includegraphics[width=\columnwidth]{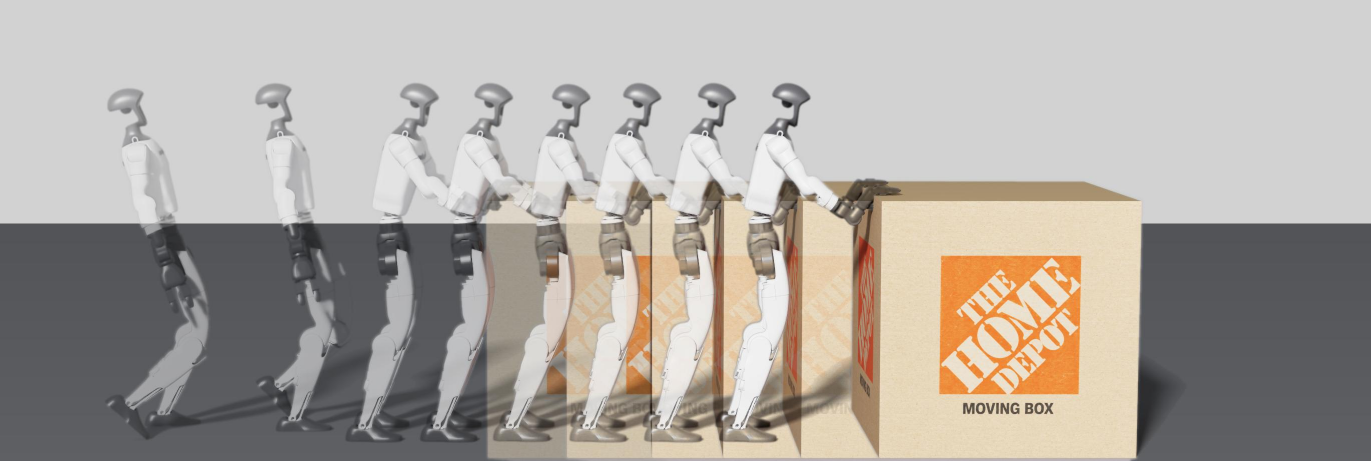}\\[2pt]
    \includegraphics[width=\columnwidth]{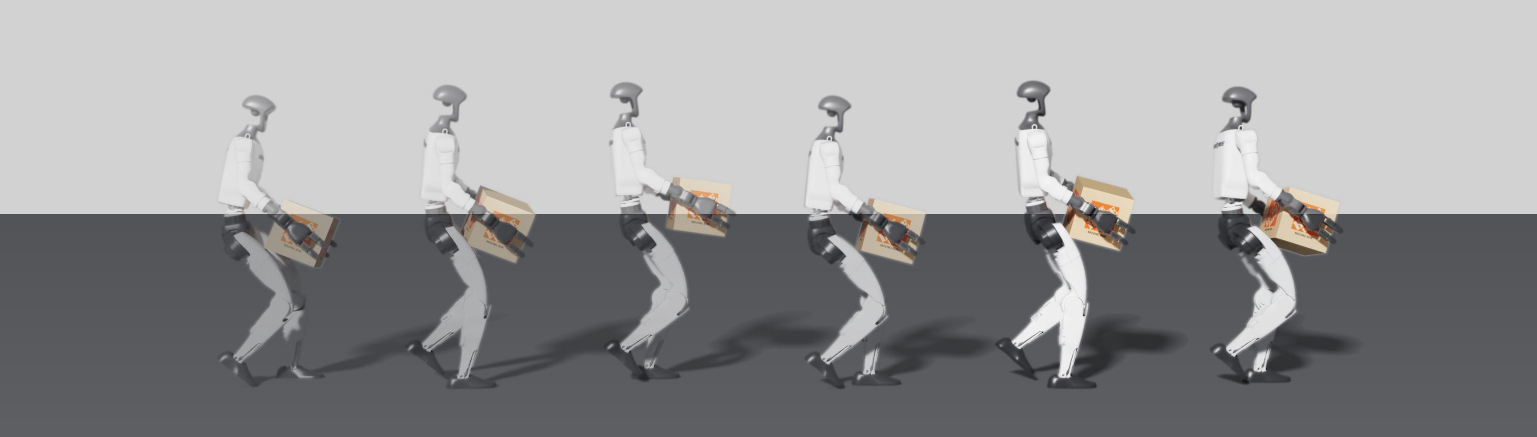}\\[2pt]
    \includegraphics[width=0.49\columnwidth]{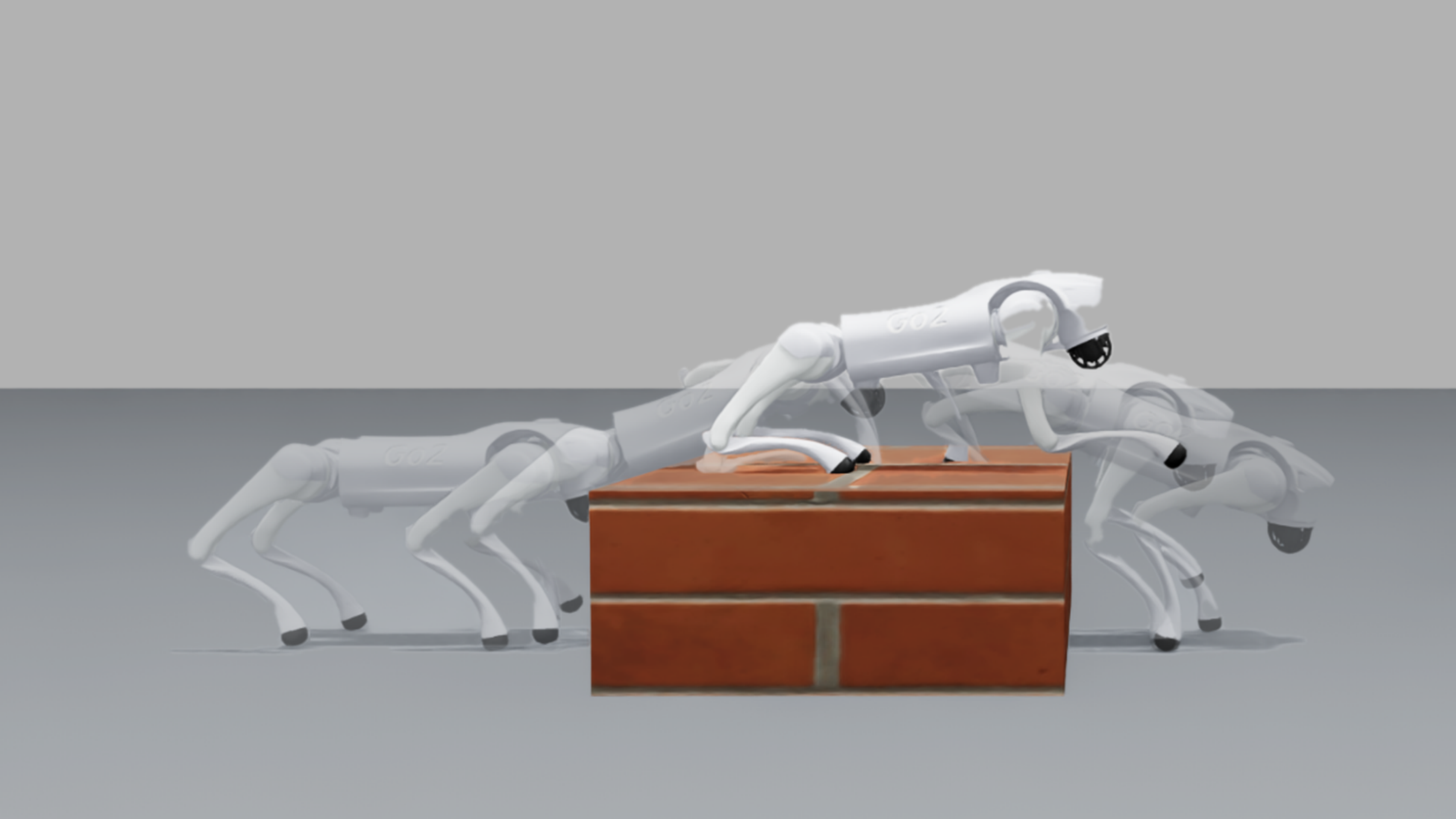}\hfill
    \includegraphics[width=0.49\columnwidth]{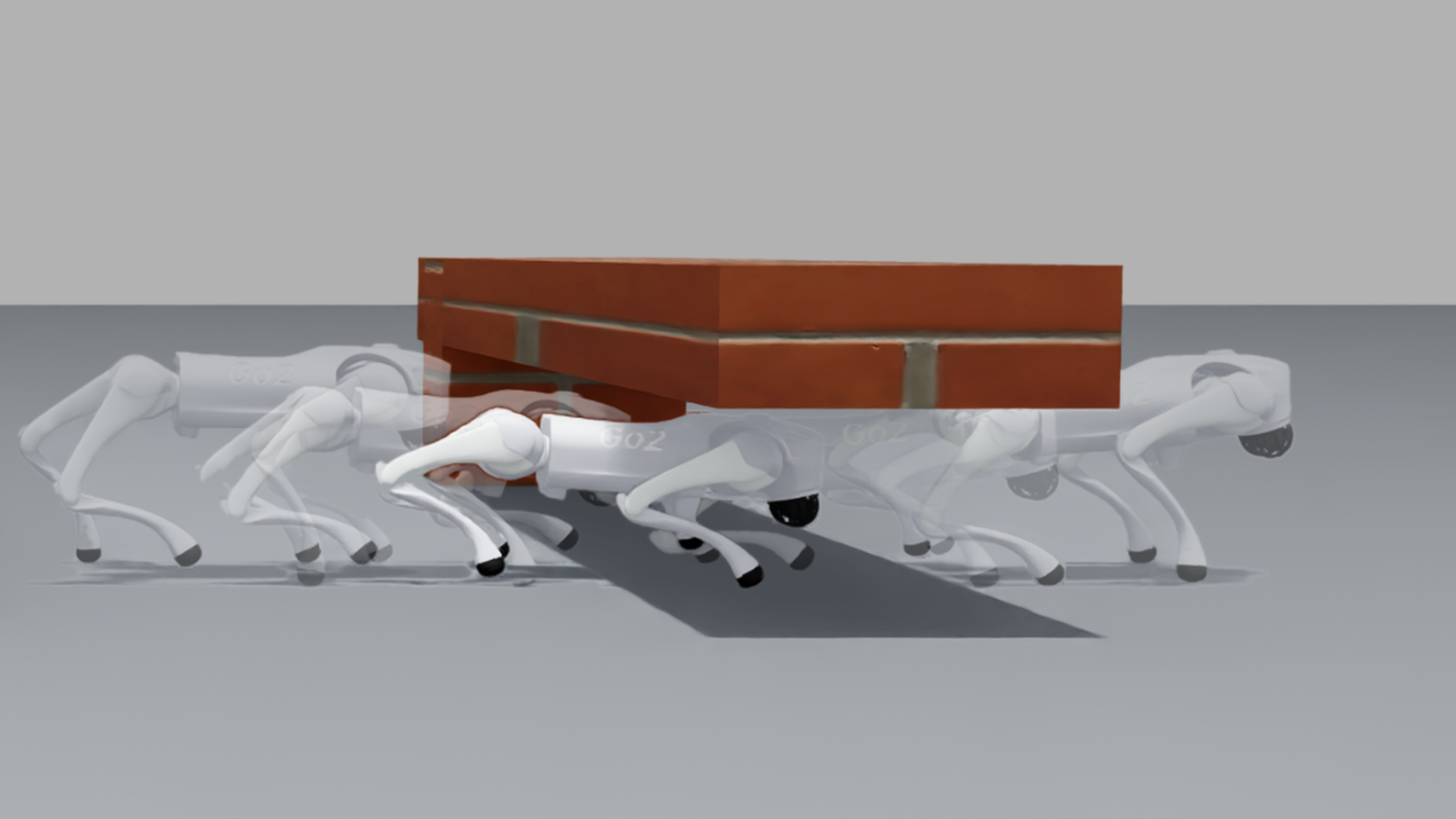}
    \caption{\textbf{SGPS policies across embodiments.}
    Overlaid poses from simulation rollouts.
    \emph{Top:} a qualitative G1 demonstration composes a trot approach
    segment with six crate-push bouts to move an $18$\,kg crate
    $1.4$\,m.
    \emph{Middle:} a Unitree G1 humanoid jogs while carrying a $4$\,kg
    box with both hands.
    \emph{Bottom left:} a Unitree Go2 clears a hurdle.
    \emph{Bottom right:} the Go2 lowers its body to crawl under a guard rail.}
    \label{fig: teaser}
\end{figure}

\section{Related Work}\label{sec: related work}

\paragraph{Learning legged locomotion}
Model-free RL trains locomotion policies in simulation~\cite{hwangbo2019learning, tan2018simtoreal, rudin2022learning}, with domain randomization supporting hardware transfer~\cite{tobin2017domain, peng2018simtoreal}.
To train policies that use onboard observations, privileged teachers can supervise students~\cite{lee2020learning, miki2022learning, kumar2021rma, chen2019lbc}, enabling perception-based locomotion and parkour~\cite{zhuang2023parkour, cheng2024extreme, hoeller2024anymal, agarwal2022legged}.
To shape the resulting behavior, training can use task-specific rewards~\cite{margolis2022walk} or motion imitation~\cite{peng2018deepmimic, peng2020animals}.
Following the reference-guided approach, SGPS uses references generated by sampling MPC on the training model.

\paragraph{Policy learning with differentiable simulation}
Differentiable simulators support first-order policy gradients~\cite{brax2021github, qiao2021efficient, mora2021pods}.
For example, SHAC~\cite{xu2021shac} combines short-horizon return gradients with a terminal value estimate; subsequent work studies locomotion~\cite{song2024learning} and optimization stability~\cite{xing2024stabilizing}.
Beyond state-based control, differentiable simulation and rendering also support visual policies~\cite{wiedemannwueest2023apg, liu2024differentiablerobotrendering, heeg2024quadrotot_vis_apg, luo2024residualpolicylearningperceptive}.
To avoid differentiating rendering, D.Va~\cite{you2025dva} uses decoupled FoPG.
Building on this formulation, SGPS adds sampling-based refinement to supply action targets during training.

\paragraph{Sampling-based MPC and distillation}
Sampling-based MPC evaluates candidate rollouts without differentiating the dynamics or reward~\cite{williams2017information, howell2022predictivesamplingrealtimebehaviour, pinneri2020icem}.
For locomotion, DIAL-MPC uses annealed sampling~\cite{xue2024dialmpc}; related work addresses hybrid modes and contact interactions~\cite{liu2025hybrid, liu2026samplebasedhybridmodecontrol}.
However, these methods require online optimization and state estimation: DIAL-MPC uses motion capture in hardware experiments, whereas sample-based hybrid mode control uses an onboard EKF.
To transfer planned behaviors into policies, guided policy search connects trajectory optimization with policy learning~\cite{levine2013gps, mordatch2014combining}.
Similarly, MPC supervision and dataset aggregation provide action targets~\cite{zhang2016mpcgps, ross2011reduction}, including for visual students~\cite{chen2019lbc, loquercio2021learning, mu2025state2vis_dagger}.
SGPS builds on this connection by adding recurring sampling-based action-target refinement to differentiable reference tracking, then distills specialists for deployment without online MPC or reference inputs.

\section{Background}\label{sec: background}

\paragraph{Problem formulation}
We consider differentiable dynamics $\mathbf{s}_{t+1}=f(\mathbf{s}_t,\mathbf{a}_t)$ and observations $\mathbf{o}_t=g(\mathbf{s}_t)$, where $g$ need not be differentiable.
The policy uses a reparameterized tanh-squashed Gaussian:
\begin{equation}
\begin{aligned}
\mathbf{z}_t&=\boldsymbol{\mu}_{\boldsymbol{\theta}}(\mathbf{o}_t)
+\boldsymbol{\sigma}_{\boldsymbol{\theta}}(\mathbf{o}_t)\odot\boldsymbol{\epsilon}_t,\\
\mathbf{a}_t&=\tanh(\mathbf{z}_t)\in[-1,1]^{d_a},
\qquad \boldsymbol{\epsilon}_t\sim\mathcal{N}(\mathbf{0},\mathbf{I}).
\end{aligned}\label{eq: policy reparameterization}
\end{equation}
The objective is to maximize expected discounted return over initial states and policy noise:
\begin{equation}
\max_{\boldsymbol{\theta}}\mathcal{V}(\boldsymbol{\theta})
=\max_{\boldsymbol{\theta}}\mathbb{E}\!\left[\sum_{t=0}^{T-1}\gamma^t R(\mathbf{s}_t,\mathbf{a}_t)\right].
\label{eq: policy objective}
\end{equation}

\paragraph{Decoupled first-order policy gradients}
FoPG differentiates returns through the simulator.
To limit instability over long horizons~\cite{metz2021gradients, suh2022differentiable}, SHAC~\cite{xu2021shac} uses $h$-step windows and a terminal critic:
\begin{equation}
\begin{split}
\mathcal{L}_{\boldsymbol{\theta}}=-\frac{1}{Nh}\sum_{i=1}^N\Big[&\sum_{t=0}^{h-1}\gamma^t R(\mathbf{s}_t^{(i)},\mathbf{a}_t^{(i)})\\
&+\gamma^h V_\phi(\mathbf{s}_h^{(i)})\Big].
\end{split}\label{eq: shac actor loss}
\end{equation}
Here $N$ is the number of parallel rollouts, time is local to each window, and $V_\phi$ is fit to TD-$\lambda$ targets.
The dynamics and reward must be differentiable along these windows.
D.Va~\cite{you2025dva} instead evaluates the actor at $\mathrm{sg}[\mathbf{o}_t]$, where $\mathrm{sg}$ denotes stop-gradient.
This excludes observation construction, including rendering, from the computation graph while retaining gradients through actions and dynamics to the policy, including its visual encoder.
The resulting update is a quasi-gradient that omits derivatives through $g$.

\paragraph{Sampling-based MPC}
MPPI-style control samples candidate action plans, weights them by exponentiated rollout returns, and executes the first action of the weighted plan before replanning~\cite{williams2017information}.
DIAL-MPC~\cite{xue2024dialmpc} represents plans by spline nodes and anneals sampling noise across iterations and along the horizon.
These searches require only forward simulation and allow nondifferentiable reference objectives.

\section{Sampling-Guided Policy Search}\label{sec: method}

\begin{figure*}[tb]
    \centering
    \includegraphics[width=\textwidth]{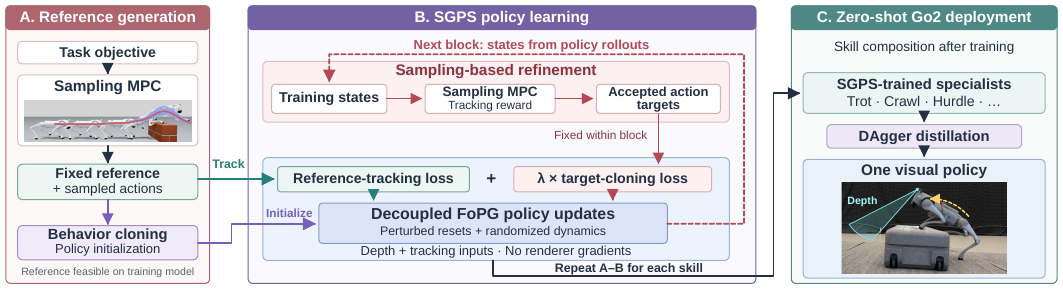}
    \caption{\textbf{SGPS learning and deployment.}
    (A) Sampling MPC supplies a fixed tracking reference and actions for behavior-cloning initialization.
    (B) Each block starts with sampling-based refinement, followed by decoupled FoPG updates combining tracking and cloning of accepted targets.
    (C) Distillation yields one Go2 policy with autonomous visual behavior transitions, onboard sensing, and velocity-only user input, without reference inputs or online MPC.}
    \label{fig: pipeline}
\end{figure*}

\subsection{Overview}\label{sec: overview}
The central component of SGPS is recurring sampling-based refinement: search supplies action targets that complement local FoPG tracking updates (Fig.~\ref{fig: pipeline}, Algorithm~\ref{alg: sgps}).
Sampling MPC first provides a reference and BC initialization; training then alternates refinement and policy updates under randomized dynamics and perturbed resets.

\subsection{Reference generation by annealed sampling}\label{sec: reference generation}
Following DIAL-MPC~\cite{xue2024dialmpc}, we represent actions $0,\ldots,H_s$ by $M+1$ nodes decoded through a quadratic spline.
At iteration $i$, candidate $n$ perturbs nominal node $m$ as
\begin{equation}
\begin{aligned}
\mathbf{y}^{(n)}_m&=\operatorname{clip}(\bar{\mathbf{y}}_m+
\sigma_{i,m}\boldsymbol{\xi}^{(n)}_m,-1,1),\\
\sigma_{i,m}&=\beta_{\mathrm{iter}}^{\,i}\beta_{\mathrm{hor}}^{\,M-m},
\qquad \boldsymbol{\xi}^{(n)}_m\sim\mathcal{N}(\mathbf{0},\mathbf{I}).
\end{aligned}\label{eq: node sampling}
\end{equation}
We fix $\boldsymbol{\xi}^{(n)}_0=\mathbf{0}$ for continuity with the preceding plan.
The factors $\beta_{\mathrm{iter}},\beta_{\mathrm{hor}}\in(0,1]$ anneal exploration across iterations and toward the start of the horizon.
Clipping bounds the nodes; each decoded candidate contains $H_s+1$ transitions.

We include the nominal plan as candidate $n=0$ and evaluate each candidate's mean task reward $r_n$ on the training model.
With $s_R=\max(\operatorname{Std}_{\ell=0:N_s}[r_\ell],\epsilon_R)$, the update is
\begin{equation}
\begin{aligned}
 w_n=\frac{\exp((r_n-r_0)/(\eta s_R))}
 {\sum_{\ell=0}^{N_s}\exp((r_\ell-r_0)/(\eta s_R))},\\
 \bar{\mathbf{y}}_m\leftarrow\sum_{n=0}^{N_s}w_n\mathbf{y}^{(n)}_m.
\end{aligned}
 \label{eq: sampling update}
\end{equation}
Here $\eta>0$ is the temperature and $\epsilon_R>0$ prevents division by zero.
The controller executes the first action, shifts and refits the spline, and replans.
The recorded trajectory $\bar{\tau}=\{(\bar{\mathbf{s}}_t,\bar{\mathbf{a}}_t)\}_{t=0}^{T-1}$ becomes a reference only after passing a task-level success test fixed before training.
The sampling objective $R_{\mathrm{task}}$ need not be differentiable.

\subsection{Behavior-cloning initialization}\label{sec: warm start}
The initialization dataset $\mathcal{D}_{\mathrm{init}}$ contains each clean reference state and $K_{\mathrm{aug}}$ noisy copies,
$\hat{\mathbf{s}}_{t,k}=\bar{\mathbf{s}}_t+\boldsymbol{\epsilon}_{t,k}$, with $\boldsymbol{\epsilon}_{t,0}=\mathbf{0}$.
Each observation $g(\hat{\mathbf{s}}_{t,k})$ is paired with the same reference action $\bar{\mathbf{a}}_t$.
Cloning only clean reference observations can leave the policy vulnerable to compounding errors when its own actions move it away from the reference trajectory.
These copies broaden the observation distribution but do not provide recovery-action labels~\cite{ross2011reduction}.
We convert actions to finite pre-squash targets using
\begin{equation}
\mathbf{z}^*=\operatorname{atanh}\!\big(\operatorname{clip}(\mathbf{a}^*,-1+\epsilon_a,1-\epsilon_a)\big),
\label{eq: safe inverse tanh}
\end{equation}
with $\epsilon_a>0$.
For observation--target pairs in $\mathcal{D}$, define
\begin{equation}
\mathcal{L}_{\mathrm{BC}}(\boldsymbol{\theta};\mathcal{D})=
\frac{\sum_{(\mathbf{o},\mathbf{z}^*)\in\mathcal{D}}
\|\boldsymbol{\mu}_{\boldsymbol{\theta}}(\mathrm{sg}[\mathbf{o}])-\mathbf{z}^*\|_2^2}
{d_a\max(|\mathcal{D}|,1)}.
\label{eq: warmstart loss}
\end{equation}
This loss averages over pairs and action dimensions and is zero for an empty dataset.
Minimizing it on $\mathcal{D}_{\mathrm{init}}$ initializes the policy.

\subsection{Policy optimization with sampling-based refinement}\label{sec: tracking}
After BC initialization, we alternate sampling-based action-target refinement with FoPG tracking updates.
We organize training into $J$ blocks, each comprising a refinement step followed by $B$ actor--critic epochs using the accepted targets.

\paragraph{Reference tracking}
Specialist observations combine proprioception, reference joint pose and velocities, body-frame base error, previous action, and a phase clock.
The differentiable tracking reward is
\begin{equation}
 R_{\mathrm{track}}(\mathbf{s}_t,\varphi(t))=
 w_{\mathrm{alive}}+\sum_k w_k\exp[-c_k d_k(\mathbf{s}_t,\bar{\mathbf{s}}_{\varphi(t)})],
 \label{eq: tracking reward}
\end{equation}
where $\varphi(t)$ selects the reference frame, $d_k\geq0$ measures a tracking error, $c_k>0$ sets its scale, and $w_k\geq0$ weights its contribution.
Terms cover body and base pose, base twist and velocity, foot height, and joint angles.
Both FoPG and refinement use this reward, whereas initial reference generation uses $R_{\mathrm{task}}$.
Decoupled rollouts evaluate Eq.~\eqref{eq: policy reparameterization} at detached observations and use Eq.~\eqref{eq: shac actor loss} with transition reward
$R(\mathbf{s}_t,\mathbf{a}_t)=R_{\mathrm{track}}(\mathbf{s}_{t+1},\varphi(t+1))$.
An asymmetric critic may receive privileged training state.

\paragraph{Sampling-based refinement}\label{sec: refinement}
At each block, we select $N_g$ current training states with their reference indices and previous actions.
For a state-based policy, we roll out the current actor deterministically from each selected state:
\begin{equation}
\begin{aligned}
\mathbf{a}^{\mathrm{seed}}_t&=\tanh\!\left(\boldsymbol{\mu}_{\boldsymbol{\theta}_j}(g(\mathbf{s}_t))\right),\\
\mathbf{s}_{t+1}&=f(\mathbf{s}_t,\mathbf{a}^{\mathrm{seed}}_t),\qquad t=0,\ldots,H_g.
\end{aligned}\label{eq: refinement seed}
\end{equation}
We fit the resulting action sequence to spline nodes and decode it to form the nominal plan.
Its score is measured by rolling out the decoded plan, accounting for the change introduced by spline fitting.
For a visual policy, nominal plans contain zero actions; search evaluates simulator states, actions, and rewards without rendering or executing the visual actor.
Each segment contains $H_g+1$ transitions.
Sampling MPC searches around each nominal plan using the annealed update above with a scaled noise level.

On the nominal training model, the mean transition tracking score of a segment is
\begin{equation}
\bar R_{\mathrm{train}}(\tau)=\frac{1}{H_g+1}\sum_{t=0}^{H_g}
R_{\mathrm{track}}(\mathbf{s}_{t+1},\varphi(t+1)).
\label{eq: refinement score}
\end{equation}
Let $\tau_{\mathrm{seed}}$ denote the decoded nominal-plan rollout.
A returned segment $\tilde{\tau}$ is accepted only if
\begin{equation}
\bar R_{\mathrm{train}}(\tilde{\tau})\geq
\max\big(\kappa,\bar R_{\mathrm{train}}(\tau_{\mathrm{seed}})+\delta\big),
\label{eq: refine gate}
\end{equation}
where the minimum score $\kappa$ and margin $\delta\geq0$ are fixed before training.
This test concerns nominal-model segment scores and does not guarantee improved policy return.

For each accepted segment, we pair the pre-action observation $g(\tilde{\mathbf{s}}_t)$ with the returned action $\tilde{\mathbf{a}}_t$.
Equation~\eqref{eq: safe inverse tanh} converts these actions to fixed pre-squash targets in the buffer $\tilde{\mathcal{D}}_j$.
The targets supervise the policy at states visited by the refined plans, while FoPG updates optimize closed-loop tracking under perturbed initial states and randomized dynamics.
The actor minimizes
\begin{equation}
\mathcal{L}_{\mathrm{SGPS}}^{(j)}=
\mathcal{L}_{\boldsymbol{\theta}}^{\mathrm{track}}+
\lambda_j\mathcal{L}_{\mathrm{BC}}(\boldsymbol{\theta};\tilde{\mathcal{D}}_j),
\label{eq: sgps actor loss}
\end{equation}
with $\lambda_j=\lambda_0(1-j/J)$ for $j=0,\ldots,J-1$.
The cloning weight decreases from $\lambda_0$ to $\lambda_0/J$, reducing the contribution of fixed action targets as policy optimization proceeds.
Targets remain fixed for all $B$ epochs and are replaced at the next block.
Refinement runs even when no segment is accepted; an empty buffer contributes zero cloning loss.
The initial reference $\bar{\tau}$ remains fixed: refinement updates action targets, not the tracking trajectory.
Recurring target supervision thus complements FoPG while preserving the behavior objective.

\subsection{Visual policy learning}\label{sec: vision}
The same algorithm trains a visual specialist directly, without a privileged
teacher policy. The specialist receives a stack of egocentric depth
images encoded by a small convolutional network~\cite{yarats2021drqv2},
together with proprioceptive and reference-conditioned tracking observations.
The critic may retain
privileged state during training but is discarded at deployment.
The distilled Go2 generalist instead uses onboard observations and a velocity
command, without reference inputs (Section~\ref{sec: generalist}).

Depth is rendered from detached simulator states while physics remains
differentiable, training the visual encoder and actor without renderer
derivatives. For behavior
cloning, noise-augmented states can share the same clean depth stack. The actor
could then use image identity instead of the perturbed proprioceptive state.
We reduce this shortcut by randomly zeroing the depth stack for a fraction of
the initialization batch. After sampling-based refinement, depth frames are
rendered for the returned trajectories, and the acceptance mask selects the
pairs that contribute to $\tilde{\mathcal{D}}_j$.

\begin{algorithm}[tb]
\caption{Sampling-Guided Policy Search (SGPS)}\label{alg: sgps}
\KwIn{simulator $f$, observation model $g$, task reward $R_{\mathrm{task}}$,
tracking reward $R_{\mathrm{track}}$, thresholds $(\kappa,\delta)$,
initial weight $\lambda_0$, $J$ blocks of $B$ epochs}
\KwOut{closed-loop policy $\pi_{\boldsymbol{\theta}}$}
Generate and validate reference $\bar{\tau}$ with sampling MPC\;
Initialize actor by cloning $\mathcal{D}_{\mathrm{init}}$ (Eq.~\eqref{eq: warmstart loss})\;
\For{$j=0,\ldots,J-1$}{
    Select training states; seed plans with state-policy rollouts or zero actions for visual policies\;
    Refine plans with sampling MPC using $R_{\mathrm{track}}$ on the nominal model\;
    Form fixed target buffer $\tilde{\mathcal{D}}_j$ from segments passing Eq.~\eqref{eq: refine gate}; render accepted visual observations\;
    Set $\lambda_j=\lambda_0(1-j/J)$\;
    \For{$B$ actor--critic epochs}{
        Collect $h$-step decoupled rollouts with randomized dynamics and perturbed reference-state resets\;
        Update actor with Eq.~\eqref{eq: sgps actor loss}; fit critic\;
    }
}
\end{algorithm}

\subsection{Connection to decoupled FoPG}\label{sec: search target distillation}
Decoupled FoPG admits an interpretation as fitting action targets obtained by a return-gradient step~\cite{you2025dva}.
SGPS uses targets from sampling-based search. The following identity relates the target-fitting direction to the policy parameters.

\begin{proposition}[Search-target distillation]\label{prop: search target distillation}
Fix observations $\mathbf{O}$ and let $\mathbf{Z}_{\boldsymbol{\theta}}(\mathbf{O})$ collect deterministic pre-squash policy outputs.
At parameters $\boldsymbol{\theta}_0$, write $\mathbf{Z}_0=\mathbf{Z}_{\boldsymbol{\theta}_0}(\mathbf{O})$ and
$\mathbf{J}_0=\left.\partial\mathbf{Z}_{\boldsymbol{\theta}}/\partial\boldsymbol{\theta}\right|_{\boldsymbol{\theta}_0}$.
For fixed targets $\tilde{\mathbf{Z}}$, define
\begin{equation}
\mathcal{L}_{\mathcal{T}}(\boldsymbol{\theta})=
\frac12\left\|\mathbf{Z}_{\boldsymbol{\theta}}(\mathbf{O})-
\mathrm{sg}[\tilde{\mathbf{Z}}]\right\|_2^2.
\label{eq: search target loss}
\end{equation}
Its gradient-descent direction at $\boldsymbol{\theta}_0$ is
\begin{equation}
-\nabla_{\boldsymbol{\theta}}\mathcal{L}_{\mathcal{T}}(\boldsymbol{\theta}_0)
=\mathbf{J}_0^\top(\tilde{\mathbf{Z}}-\mathbf{Z}_0).
\label{eq: search target direction}
\end{equation}
Let $\mathcal{J}(\mathbf{Z})$ be the return obtained by executing $\tanh(\mathbf{Z})$ from a fixed initial state.
If $\tilde{\mathbf{Z}}=\mathbf{Z}_0+\beta\nabla_{\mathbf{Z}}\mathcal{J}(\mathbf{Z}_0)$, with $\beta>0$, then
\begin{equation}
-\frac{1}{\beta}\nabla_{\boldsymbol{\theta}}\mathcal{L}_{\mathcal{T}}(\boldsymbol{\theta}_0)
=\mathbf{J}_0^\top\nabla_{\mathbf{Z}}\mathcal{J}(\mathbf{Z}_0),
\label{eq: gradient target identity}
\end{equation}
which is the deterministic decoupled FoPG direction.
\end{proposition}
\begin{proof}
Holding observations and targets fixed, the chain rule gives
$\nabla_{\boldsymbol{\theta}}\mathcal{L}_{\mathcal{T}}(\boldsymbol{\theta}_0)
=\mathbf{J}_0^\top(\mathbf{Z}_0-\tilde{\mathbf{Z}})$.
Negating yields Eq.~\eqref{eq: search target direction}; substituting the gradient-step target gives Eq.~\eqref{eq: gradient target identity}.
\end{proof}

The equivalence holds at the generating parameters; subsequent finite steps toward fixed targets need not equal FoPG updates.
Equation~\eqref{eq: warmstart loss} scales the target-fitting direction by $2/(d_a|\mathcal{D}|)$ for a nonempty dataset.
Sampling targets need not follow a return gradient, and their acceptance on the nominal model does not guarantee an increase in the policy's expected return after fitting.

\section{Experiments}\label{sec: experiment}
We evaluate Go2 trot and visual obstacle traversal, and G1 pushing and carrying.
We then distill Go2 specialists for autonomous visual obstacle traversal on hardware.

\subsection{Experimental Setup}\label{sec: setup}
\paragraph{Tasks and policy interfaces}
Go2 and G1 actors output 12 and 29 joint-position commands at $50$\,Hz, respectively, using
$\mathbf{q}^{\mathrm{cmd}}=\mathbf{q}_{\mathrm{anchor}}+\mathbf{d}\odot\mathbf{a}$.
Go2 uses its home pose as the anchor and repeats $\mathbf{d}=[0.4,0.8,0.8]$ for each leg; G1 uses task-specific anchors and scales.
The specialist tracking observation is
\begin{equation}
\begin{split}
\mathbf{o}^{\mathrm{track}}_t=[\,&\mathbf{g}^B_t,
\mathbf{q}_{t,7:}-\mathbf{q}_{\mathrm{home}},\mathbf{u}_{t-1},
\bar{\mathbf{q}}_{t,7:},\mathbf{e}^B_{p,t},\\
&\mathbf{v}^B_t,\bar{\mathbf{v}}^B_t,
\boldsymbol{\omega}^B_t,\sin\phi_t,\cos\phi_t\,].
\end{split}\label{eq: tracking observation}
\end{equation}
Here $B$ denotes the base or pelvis frame, bars denote the reference,
$\mathbf{u}_{t-1}=\mathbf{q}^{\mathrm{cmd}}_{t-1}$ is the previous joint command,
and $\mathbf{e}^B_{p,t}$ is the reference base-position error.
The vectors have 53 dimensions for Go2, 106 for G1 pushing, and 118 for G1 carrying.
G1 inputs include current and reference object coordinates.
Visual specialists fuse these vectors with three $64\times64$ egocentric depth frames; they therefore use depth together with tracking inputs.

The G1 simulated state contains the floating base, 29 actuated joints, and
the manipulated object. Pushing uses
$(\mathbf{q},\dot{\mathbf{q}})\in\mathbb{R}^{37}\times\mathbb{R}^{36}$,
with one crate slide coordinate; carrying uses
$\mathbb{R}^{43}\times\mathbb{R}^{41}$, with a free box joint.

Reference generation uses $N_s{=}2048$ sampled
action candidates in addition to the nominal plan, sampling temperature
$\eta{=}0.05$, and annealing factors
$(\beta_{\mathrm{hor}},\beta_{\mathrm{iter}})=(0.9,0.5)$. Sampling MPC runs two
iterations per control step and ten for the initial zero plan. Go2 trot uses six
spline nodes over a 20-step lookahead; the remaining Go2 skills use nine nodes
over 16 steps. G1 object-task reference generation retains the same
annealed-search procedure but uses task-specific action anchors, action scales,
and reward specifications.

The task reward specifies gait phase, duty ratio, cadence, swing amplitude,
base velocity, height, and uprightness. Obstacle tasks additionally use
position-indexed base-height waypoints and world-frame velocity tracking.
These terms specify the behavior sought by sampling MPC; policy optimization
uses the common tracking reward in Eq.~\eqref{eq: tracking reward}. The G1
push task adds explicit crate-progress tracking because base and joint terms
alone do not distinguish pushing from standing; the carry reference includes
the free-box pose so the common tracking objective covers the loaded system.

\paragraph{Training details}
We use MuJoCo MJX for simulation and train policies on a single NVIDIA
RTX 4080 GPU.
Initialization uses eight noisy copies of every clean reference state. FoPG
uses $h{=}32$-step windows in 64 parallel environments, reference-state
initialization, and per-environment dynamics randomization. Training uses
$J{=}25$ blocks, each with refinement from $N_g{=}16$ current training states
followed by $B{=}80$ actor--critic epochs. Refinement uses $26$-transition segments
($H_g{=}25$), six spline nodes, two iterations, $N_s{=}64$ sampled candidates,
and an initial noise scale of $0.1$. We set $\lambda_0{=}1$ and
use the linear schedule
$\lambda_j=\lambda_0\max(0,1-j/J)$. The standard budget is $2000$ epochs ($4.096$M FoPG rollout interactions, plus search rollouts). Figure~\ref{fig: carry vision curves} also shows longer carrying runs for comparison with DAgger. We train three seeds per G1 task and actor modality.

\paragraph{Baselines and plotted evaluation}
We compare against BC only, FoPG from scratch, PPO, and state-to-vision
DAgger. BC only stops after initialization, whereas FoPG from scratch starts
from random parameters. DAgger first trains a state expert, then distills a
visual student. All curves report deterministic evaluations on the nominal
model. For computational comparisons, training times include warm starts,
refinement search, and, for DAgger, expert training. Environment-step counts
also include search rollouts, accounting for refinement's cost in both
comparisons.

\subsection{Behavior Selection and Learning Efficiency}
\label{sec: behavior selection}\label{sec: trot results}
\paragraph{Primary locomotion setting}
The primary benchmark asks the Go2 to track a forward trot at $0.8$\,m/s on
flat terrain. The reference specifies diagonal foot pairs in phase, a $0.45$
duty ratio, a $2$\,Hz cadence, a $0.08$\,m swing height, and a $0.30$\,m base
height. Both SGPS variants, PPO, FoPG from scratch, and DAgger use three
training seeds.

\begin{figure}[t]
    \centering
    \includegraphics[width=\columnwidth]{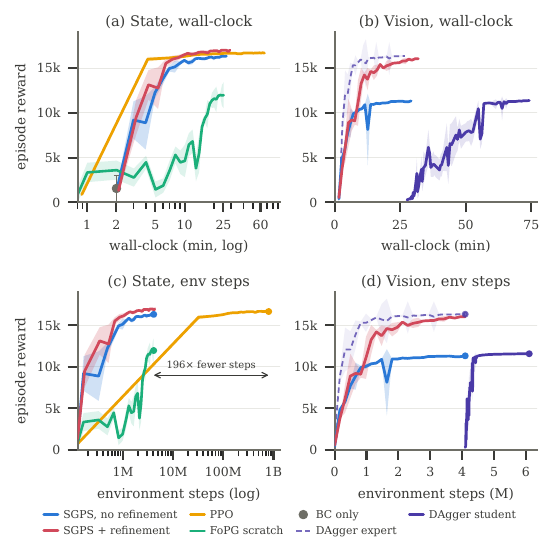}
    \caption{\textbf{Trot learning efficiency and refinement ablation.}
    Episode return versus training time (top) and environment interactions
    (bottom), for state (left) and visual (right) policies.
    Red includes recurring search and target cloning; blue retains BC
    initialization and FoPG tracking without refinement.
    SGPS shading shows $\pm1$ standard deviation across seeds.}
    \label{fig: single skill results}
\end{figure}

Figure~\ref{fig: single skill results} shows that full SGPS achieves high
returns with both state and visual observations.
For state-based learning, SGPS reaches a higher return than FoPG from scratch
over the reported budgets, while PPO uses $196$ times as many environment
interactions as SGPS.
For visual learning, SGPS trains directly after initialization, whereas
DAgger first requires a separately trained state expert.

\paragraph{Behavior selection}
\label{sec: gait library}
To illustrate a failure under a gait-blind reward, we train a policy from
scratch with the joint-angle tracking weight $w_{\mathrm{joint}}$ set to zero.
This removes direct joint-pose supervision while keeping the other tracking
terms unchanged.

\begin{figure}[t]
    \centering
    \includegraphics[width=\columnwidth]{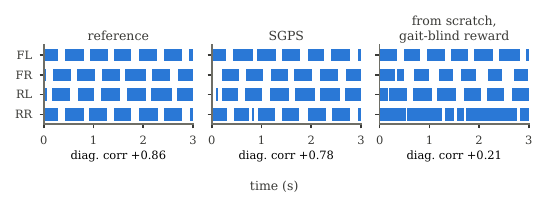}
    \caption{\textbf{Go2 foot-contact patterns.}
    Left: sampling MPC reference. Center: SGPS. Right: FoPG trained from
    scratch without joint-angle tracking.
    Blue intervals indicate ground contact for the front-left (FL),
    front-right (FR), rear-left (RL), and rear-right (RR) feet.
    Values report mean diagonal-pair contact correlation (FL--RR and FR--RL).}
    \label{fig: footfall}
\end{figure}

Figure~\ref{fig: footfall} shows alternating diagonal foot contacts in both
the sampled reference and SGPS, with mean diagonal-pair correlations of
$+0.86$ and $+0.78$, respectively.
The gait-blind policy trained from scratch instead produces a weakly
coordinated shuffle ($+0.21$).
This illustrates a failure to recover the intended trot without BC
initialization or joint-angle supervision.
Figure~\ref{fig: single skill results} separately evaluates refinement with
these settings held fixed.

\subsection{Ablation of Sampling-Based Refinement}
\label{sec: refinement results}\label{sec: ablations}
We test the contribution of recurring refinement beyond initialization and
FoPG tracking.

\paragraph{Controlled comparison}
``SGPS + refinement'' denotes the full method. ``SGPS, no refinement''
retains BC initialization and reference tracking, but omits recurring
search and accepted-target cloning in Eq.~\eqref{eq: sgps actor loss}.
Both variants use three training seeds and share all other settings,
including reference generation, BC initialization, architecture, tracking
reward, dynamics randomization, and FoPG optimization.
The ablation tests recurring search and target supervision jointly,
without separating the gate, search initialization, or cloning schedule.

\paragraph{Go2 trot}
For state-based trot, both SGPS variants reach high returns, with a modest
gain from refinement (Fig.~\ref{fig: single skill results}a,c).
Thus, much of the efficiency relative to PPO remains with sampled-reference
initialization and tracking alone.
The visual comparison shows a larger gap: near $4.1$M interactions, full
SGPS reaches approximately $16{,}000$, while the variant without refinement
levels off near $11{,}000$ (Fig.~\ref{fig: single skill results}d).
A comparable gap appears around $30$ minutes of training time,
including refinement search.

\paragraph{G1 visual carrying}
Figure~\ref{fig: carry vision curves} extends the comparison to carrying.
Refinement again improves visual-policy return: near $6$M interactions,
full SGPS reaches approximately $26{,}500$, compared with about $20{,}500$
without refinement. The gap also persists at equal displayed training
times: around $100$ minutes, the two variants reach approximately
$26{,}500$ and $21{,}000$, respectively.

\subsection{Perception-Aware Obstacle Tasks}
\label{sec: vision results}\label{sec: obstacles}
To evaluate visual policy learning beyond flat-ground locomotion, we consider
crawling beneath an overhead beam and clearing a hurdle. These tasks require
changes in body height and foot motion around obstacles. Both retain the Go2
action map and 53-dimensional tracking vector. Obstacle geometry is observed
through the three-frame depth stack, while the sampled reference specifies
the target motion. SGPS trains these visual policies directly, without a
separate state-policy teacher.

\paragraph{Crawl}
The robot approaches an overhead beam while tracking a $0.8$\,m/s trot.
To pass beneath it, the robot must lower its body, continue forward, and
then rise to resume trotting. The beam occupies $x\in[1.85,2.15]$\,m and
has an underside height of $0.30$\,m. Accordingly, the sampled reference
lowers the base-height target from $0.30$\,m to $0.20$\,m over
$x\in[1.55,2.45]$\,m while keeping the feet ground-referenced.

\paragraph{Hurdle}
The robot approaches a $0.10$\,m-high bar at $x=2.0$\,m under a
$1.2$\,m/s gallop specification. To clear the bar while maintaining forward
motion, the reference raises the base-height target from $0.30$\,m to
$0.38$\,m over $x\in[1.75,2.25]$\,m and increases the swing-foot targets.
These adjustments position takeoff and landing around the hurdle, after
which the robot should return to its locomotion gait.

Figure~\ref{fig: teaser} (bottom) illustrates both behaviors in the composed
Go2 rollout: lowering the body beneath the beam and raising it to clear
the hurdle.

\subsection{Humanoid Pushing and Carrying}
\label{sec: humanoid}

\begin{figure}[t]
    \centering
    \includegraphics[width=\columnwidth]{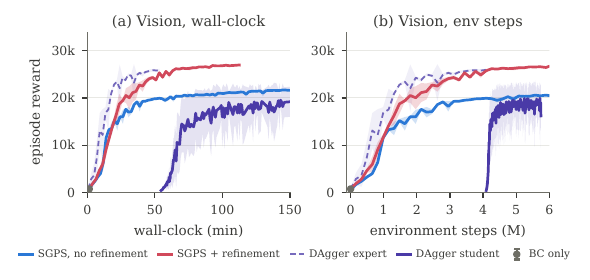}
    \caption{\textbf{G1 visual carrying and refinement ablation.}
    Episode return versus training time (a) and environment interactions (b).
    Red includes recurring refinement; blue retains initialization and
    FoPG tracking without it.
    SGPS shading shows $\pm1$ standard deviation across seeds.}
    \label{fig: carry vision curves}
\end{figure}

To evaluate SGPS on whole-body object interaction, we extend it to pushing
and carrying with the 29-DoF Unitree G1. These tasks require coordinating
locomotion with arm--object contact. We retain the same learning procedure
with task-specific dynamics, action maps, and objectives, without transferring
Go2 policy parameters. Both tasks include current and reference object
coordinates in the tracking observations, and the object is visible in the
depth images.

\paragraph{Crate pushing}
The G1 must push an $18$\,kg crate forward at $0.08$\,m/s while maintaining
an upright posture and contact with both hands. The crate is a $0.9$\,m cube
constrained to move along the forward axis, and the robot starts with its
hands $2$\,cm from the crate face. To encourage sustained pushing, the sampling
objective rewards hand--crate contact and forward progress while penalizing
torso or leg contact. Policy learning also tracks crate displacement, so
remaining stationary fails the objective.

\paragraph{Box carrying}
The G1 must follow a $0.5$\,m/s jogging reference while holding a $4$\,kg box
at chest height with both hands. The box measures
$0.28\times0.22\times0.22$\,m. Two soft wrist-to-box constraints transmit
the load while simplifying grasp acquisition and slip. To maintain the
carrying posture, the reference includes the box pose, and task-specific
action scales limit unnecessary arm motion.

Figure~\ref{fig: carry vision curves} shows the visual carrying results under
deterministic nominal-model evaluation with fixed reference-state
initializations. Within one hour, including refinement search, SGPS reaches
approximately $25{,}000$ return and successfully performs the task.
It approaches the state expert's final return within about $4.1$M interactions
and subsequently exceeds it. In comparison, the DAgger student improves later
and remains below its teacher over the displayed training range.

\subsection{Skill Composition and Hardware Evaluation}\label{sec: generalist}
For deployment, we distill the Go2 trot, crawl, and hurdle specialists into
one visual policy that transitions between behaviors using onboard
observations, without reference inputs or online MPC.

\begin{figure}[t]
    \centering
    \includegraphics[width=\columnwidth]{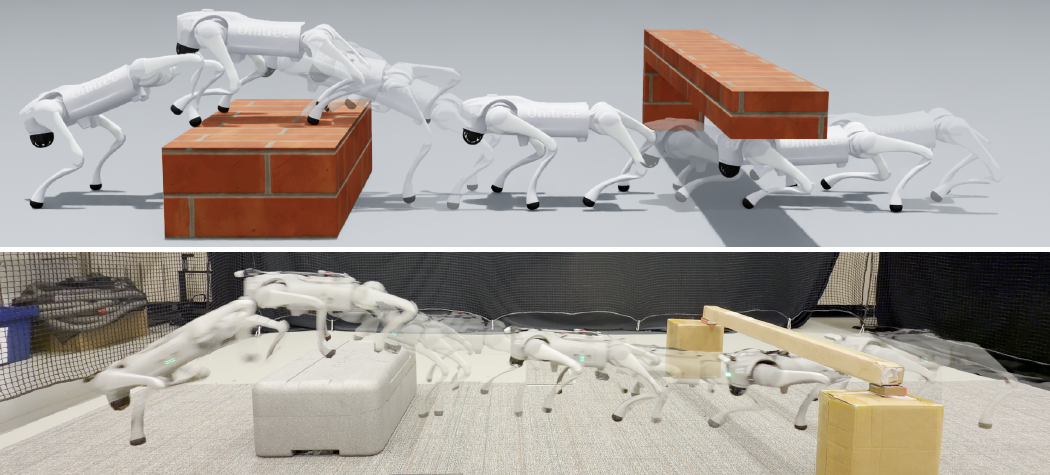}
    \caption{\textbf{Autonomous visual Go2 deployment of crawl, trot, and hurdle.}
    Simulation (top) and hardware (bottom) poses show crawling beneath a
    beam, trotting toward a box, then climbing, traversing, and landing,
    from right to left on analogous courses. Translucent poses show
    intermediate motion. The policy transfers zero-shot, with autonomous
    visual transitions and a velocity command as the only user input.}
    \label{fig: generalist hardware}
\end{figure}

\paragraph{Combining specialists}
Using DAgger~\cite{ross2011reduction}, we label states visited by the unified
policy with actions from the corresponding specialist.
Let $\mathcal{K}$ denote the specialist library,
$\bar{\rho}_i^k$ the collected state distribution for skill $k$ through
iteration $i$, and $p_k$ its sampling weight, with $\sum_k p_k=1$.
The unified policy minimizes the weighted behavior-cloning loss
\begin{equation}
\begin{split}
\boldsymbol{\theta}_{i+1}
={}&\arg\min_{\boldsymbol{\theta}}
\sum_{k\in\mathcal{K}}p_k
\mathbb{E}_{(\mathbf{s}_t,\mathbf{o}^{\mathrm{dep}}_t)\sim\bar{\rho}_i^k}
\Big[
\big\|\boldsymbol{\mu}_{\boldsymbol{\theta}}
    (\mathbf{o}^{\mathrm{dep}}_t) \\
&\qquad-
\mathrm{sg}\!\left[
\boldsymbol{\mu}_{\boldsymbol{\phi}_k}(g_k(\mathbf{s}_t))
\right]\big\|_2^2
\Big].
\end{split}
\label{eq: multiskill distillation}
\end{equation}
Here $\boldsymbol{\mu}$ denotes the policy output used for cloning and
$\mathrm{sg}$ fixes the specialist target. Deployment observations
$\mathbf{o}^{\mathrm{dep}}_t$ contain depth, proprioception, observation
history, learned state estimates, and the velocity command.
During training, specialists receive reference features through $g_k(\mathbf{s}_t)$,
while the unified policy receives no reference trajectory or phase.
The skill index $k$ selects the training supervisor; at deployment, depth
conditions behavior transitions.

To learn transitions, we collect obstacle-traversal and locomotion data.
Crawl covers the full reference window, while hurdle starts emphasize takeoff
and include post-obstacle states. We retain specialist tracking-error
termination rules to prevent stationary post-hurdle states from dominating
training.

\paragraph{Zero-shot hardware deployment}\label{sec: real world}
We deploy the unified policy on a Unitree Go2 EDU without hardware
fine-tuning. Inference runs on its onboard NVIDIA Jetson Orin module with
$8$\,GB shared memory. The robot uses an Intel RealSense D435i depth camera,
the onboard IMU, joint encoders, and foot force sensors. No sampled reference,
motion capture, external localization, or lidar odometry is used at deployment.

\emph{Observations and timing.} The control loop runs at $50$\,Hz and the
recurrent depth encoder at $10$\,Hz.
Depth is acquired at $848\times480$ and $30$\,Hz, converted from camera-axis
depth to range, cropped to the training camera's angular window, clipped at
$2$\,m, and resized to $58\times87$. The encoder uses the previous depth
frame to match the $100$\,ms training delay. These inputs differ from the
$64\times64$ specialist frame stacks.
Proprioceptive history spans ten previous frames ($200$\,ms).

\emph{Policy execution.} The deployed networks comprise a learned velocity estimator, a recurrent
depth encoder, and an actor with a proprioceptive-history encoder.
Base linear velocity is predicted from proprioception; orientation comes
from onboard IMU fusion. The actor outputs twelve joint-position
offsets, converted to targets as
$\mathbf{q}^{\mathrm{cmd}}=\mathbf{q}_{\mathrm{default}}
+0.25\operatorname{clip}(\mathbf{a},-4.8,4.8)$.
Joint PD gains are $k_p=40$\,N\,m/rad and $k_d=1$\,N\,m\,s/rad.
Onboard GPU TorchScript inference averages $2.7$\,ms ($4.3$\,ms at the
99th percentile); the complete control step averages $11$\,ms.
The training critic is discarded.

\paragraph{Autonomous behavior transitions}
Figure~\ref{fig: generalist hardware} shows crawling beneath a beam, resuming
trotting, then climbing onto and traversing a box before landing.
Transitions are autonomous, with velocity-only user input. The simulation
rollout shows the same sequence on an analogous course with a wider box.

\section{Conclusion and Limitations}\label{sec: conclusion}\label{sec: limitations}
SGPS combines recurring sampling-based action-target refinement with decoupled
FoPG, providing search-based supervision throughout training while keeping the
tracking reference fixed. Experiments show gains beyond initialization and
tracking alone. Specialist distillation yields a unified visual policy that
transfers zero-shot to a real Go2 and autonomously transitions between behaviors
using onboard observations.

However, SGPS relies on differentiable dynamics, successful sampled
references, and task-specific objectives. Moreover, the humanoid tasks simplify
object motion and hand--object interactions; complex manipulation remains
unexplored. Future work includes realistic object interactions, long-horizon
manipulation, and humanoid hardware deployment.

\section*{Acknowledgement}
We acknowledge the use of ChatGPT and Claude to assist with editing, grammar, and code generation for plots.

\IEEEtriggeratref{27}
\bibliographystyle{IEEEtran}
\bibliography{reference}

\end{document}